\documentclass{article} 
\usepackage{nips12submit_e,times}
\usepackage{amsmath,amssymb,graphicx}
\usepackage{amsthm}
\usepackage{dsfont}
\usepackage{algorithm}
\usepackage{algorithmic}
\usepackage{bm}

\newtheorem{theorem}{Theorem}
\newtheorem{proposition}[theorem]{Proposition}

\newtheorem{assumption}[theorem]{Assumption}

\theoremstyle{definition}
\newtheorem{remark}[theorem]{Remark}

\newcommand{\ucrl}{\textsc{UCCRl}}

\newcommand{\sS}{\set{S}}
\newcommand{\sA}{\set{A}}

\newcommand{\agg}{^{\rm agg}}

\newcommand{\rh}{\hat{r}}
\newcommand{\rt}{\tilde{r}}
\newcommand{\ph}{\hat{p}}
\newcommand{\pt}{\tilde{p}}

\newcommand{\rhot}{\tilde{\rho}}

\newcommand{\st}{\textstyle}

\newcommand{\meanr}{{r}}

\newcommand{\tpi}[1]{\tilde\pi_{#1}}
\newcommand{\M}{M}

\newcommand{\sM}{\mathcal{M}}
\newcommand{\tM}{\tilde M}
\newcommand{\tMk}{\tM_{k}}
\newcommand{\set}[1]{\mathcal{#1}}
\newcommand{\Ind}[1]{\mathds{1}_{#1}}

\newcommand{\tlambda}{\tilde{\lambda}}

\title{Online Regret Bounds for Undiscounted Continuous Reinforcement Learning}

\author{
Ronald Ortner$^*${$^\dagger$}  \\
$^*$Montanuniversitaet Leoben\\
8700 Leoben, Austria \\
\texttt{rortner@unileoben.ac.at} \\
\And
 Daniil Ryabko$^\dagger$\\
$^\dagger$INRIA Lille-Nord Europe, \'{e}quipe SequeL \\
59650 Villeneuve d'Ascq, France\\
\texttt{daniil@ryabko.net}
}

\nipsfinalcopy % Uncomment for camera-ready version

\begin{document}

\maketitle

\begin{abstract}
We derive sublinear regret bounds for undiscounted reinforcement learning
in continuous state space. The proposed algorithm combines state aggregation
with the use of upper confidence bounds for implementing 
optimism in the face of uncertainty.
Beside the existence of an optimal policy which satisfies the Poisson equation, 
the only assumptions made are H\"older continuity of rewards and transition probabilities. 
\end{abstract}

\section{Introduction}

Real world problems usually demand continuous state or action spaces,
  and one of the challenges for reinforcement learning is to deal with such continuous domains.
In many problems there is a natural metric on the state space such that close states exhibit similar behavior.
Often such similarities can be formalized as Lipschitz or more generally H\"older continuity of reward and transition functions.

The simplest continuous reinforcement learning problem is the 1-dimensional continuum-armed bandit,
  where the learner has to choose arms from a bounded interval.
Bounds on the regret with respect to an optimal policy under the assumption that the reward 
function is H\"older continuous have been given in \cite{kleinberg,auorsze}.
The proposed algorithms apply the UCB algorithm~\cite{acbf} to a discretization of the problem. 
That way, the regret suffered by the algorithm consists of the loss by aggregation (which can be 
bounded using H\"older continuity) plus the regret the algorithm incurs in the discretized setting.
More recently, algorithms that adapt the used discretization (making it finer in more promising regions) 
have been proposed and analyzed~\cite{upfal,remi-x}.

While the continuous bandit case has been investigated in detail,
  in the general case of continuous state Markov decision processes (MDPs) a lot of work 
  is confined to rather particular settings,  primarily with respect to the considered transition model.
 In the simplest case, the transition function is considered to be deterministic as in~\cite{neupfeima},
 and mistake bounds for the respective discounted setting have been derived in~\cite{beshim}.
Another common assumption is that
transition functions are linear functions of state and action 
plus some noise. For such settings sample complexity bounds have been given in~\cite{streli-cont,bruns},
while $\tilde{O}(\sqrt{T})$ bounds for the regret after $T$ steps are shown in~\cite{abbacse}.
However, there is also some research considering more general transition dynamics under the assumption that 
close states behave similarly, as will be considered here. While most of this work is purely experimental \cite{josto,uve},
there are also some contributions with theoretical guarantees. 
Thus, \cite{kakealang} considers PAC-learning for continuous reinforcement learning in metric state spaces,
when generative sampling is possible. The proposed algorithm is a generalization of the E$^3$ algorithm~\cite{kesi}
to continuous domains. A respective adaptive discretization approach is suggested in~\cite{nouli}.
The PAC-like bounds derived there however depend on the (random) behavior of the proposed algorithm.
\medskip
 
Here we suggest a learning algorithm for undiscounted reinforcement learning 
in continuous state space. The proposed algorithm is in the tradition
of algorithms like UCRL2~\cite{jaorau} in that it implements the ``optimism in the face of uncertainty''
maxim, here combined with state aggregation. 
%%%d
Thus, the algorithm does not need a generative model or access to ``resets:''
learning is done online, that is,  in a single continual session of interactions between the environment and the learning policy. %%%%d.

For our algorithm we derive regret bounds of 
$\tilde{O}(T^{(2+\alpha)/(2+2\alpha)})$ 
for MDPs with $1$-dimensional state space and H\"older-continuous 
rewards and transition probabilities with parameter $\alpha$. 
These bounds also straightforwardly generalize to dimension $d$
where the regret is bounded by $\tilde{O}(T^{(2d+\alpha)/(2d+2\alpha)})$. 
Thus, in particular, if rewards and 
transition probabilities are Lipschitz, the regret is bounded by 
$\tilde{O}(T^{(2d+1)/(2d+2))})$ in dimension $d$ and $\tilde{O}(T^{3/4})$ in dimension 1.
We also present an accompanying lower bound of $\Omega(\sqrt{T})$. 
As far as we know, these are the first regret bounds 
for a general undiscounted continuous reinforcement learning setting.

%%%
\section{Preliminaries}\label{sec:prel}
We consider the following setting.
Given is a Markov decision process (MDP) $M$ with state space $\sS=[0,1]^d$ and finite action space $\sA$.
For the sake of simplicity, in the following we assume $d=1$. However, proofs and 
results generalize straightforwardly to arbitrary dimension, cf.\ Remark \ref{rem:d} below. 
The random rewards in state $s$ under action $a$ are assumed to be bounded in $[0,1]$
with mean $r(s,a)$.
The transition probability distribution in state $s$ under action $a$ is denoted by 
$p(\cdot|s,a)$.

We will make the natural assumption that rewards and transition probabilities are 
similar in close states. More precisely, we assume that rewards and transition probabilities are \textit{H\"older continuous}.

\begin{assumption}\label{ass:rew}
There are $L,\alpha>0$ such that for any two states $s,s'$ and all actions $a$, 
 \[
    |r(s,a)-r(s',a)| \,\leq\, L |s-s'|^\alpha.
 \]
\end{assumption}

\begin{assumption}\label{ass:prob}
There are $L,\alpha>0$ such that for any two states $s,s'$ and all actions $a$, 
 \[
  \big\|p(\cdot|s,a)-p(\cdot|s',a)\big\|_1 \,\leq\, L |s-s'|^\alpha.
 \]
\end{assumption}

For the sake of simplicity we will assume that $\alpha$ and $L$ in Assumptions \ref{ass:rew} and \ref{ass:prob}
are the same.

We also assume existence of an optimal policy $\pi^*:\sS\to\sA$ which gives optimal average reward 
$\rho^*=\rho^*(M)$ on $M$ independent of the initial state. A sufficient condition for state-independent
optimal reward is geometric convergence of $\pi^*$ to an invariant probability measure. This is a natural condition
which e.g.\ holds for any communicating finite state MDP. It also ensures (cf.~Chapter 10 of \cite{hela42}) 
that the Poisson equation holds for the optimal policy. In general, under suitable technical conditions
(like geometric convergence to an invariant probability measure $\mu_\pi$) the \textit{Poisson equation}
\begin{equation}\label{eq:poisson}
	\rho_\pi + \lambda_\pi(s) = r(s,\pi(s)) + \int_\sS  p(ds'|s,\pi(s))\cdot\lambda_\pi(s')
\end{equation}
relates the rewards and transition probabilities under any 
measurable policy 
$\pi$ to its average reward~$\rho_\pi$ and the \textit{bias} function 
$\lambda_\pi:\sS\to \mathbb{R}$ of $\pi$.
Intuitively, the bias is the difference in accumulated rewards when starting in a different state.
Formally, the bias is defined by the Poisson equation \eqref{eq:poisson} and the normalizing equation 
$\int_\sS \lambda_\pi \, d\mu_\pi=0$ (cf.~e.g.~\cite{hela30}). The following result follows from the bias definition
and Assumptions~\ref{ass:rew} and~\ref{ass:prob} (together with results from Chapter~10 of~\cite{hela42}).

\begin{proposition}\label{ass:mix}
Under Assumptions \ref{ass:rew} and \ref{ass:prob}, the bias of the optimal policy is bounded.
\end{proposition}

Consequently, it makes sense to define the \textit{bias span} $H(M)$ of a continuous state MDP~$M$ satisfying 
Assumptions~\ref{ass:rew} and~\ref{ass:prob} to be 
$H(M):=\sup_s \lambda_{\pi^*}(s) - \inf_s \lambda_{\pi^*}(s)$.
Note that since  $\inf_s \lambda_{\pi^*}(s)\leq 0$ by definition of the bias, the bias function~$\lambda_{\pi^*}$
is upper bounded by~$H(M)$.

We are interested in algorithms which can compete with the optimal policy $\pi^*$
and measure their performance by the \textit{regret} (after $T$ steps) defined as $T\rho^*(M) - \sum_{t=1}^T r_t$,
where $r_t$ is the random reward obtained by the algorithm at step $t$.
Indeed, within $T$ steps no \textit{canonical} or even \textit{bias optimal} optimal policy (cf. Chapter~10 of~\cite{hela42})
can obtain higher accumulated reward than $T\rho^* + H(M)$.

%%%
\section{Algorithm}

Our algorithm \ucrl, shown in detail in Figure \ref{alg:col}, implements the 
``optimism in the face of uncertainty maxim'' just like UCRL2~\cite{jaorau} or REGAL~\cite{regal}.
%%%
\begin{algorithm}[tb]
   \caption{The \ucrl\ algorithm}
   \label{alg:col}
\begin{algorithmic}
    \STATE   {\bfseries Input:} State space $\sS=[0,1]$, action space $\sA$, confidence parameter $\delta>0$, 
    aggregation parameter~$n\in\mathbb{N}$, upper bound $H$ on the bias span, Lipschitz parameters $L,\alpha$. \smallskip
    \STATE   {\bfseries Initialization:}  \\
    \quad $\rhd$\quad Let $I_1:=\big[0, \frac{1}{n}\big]$, $I_j:=\big(\frac{j-1}{n}, \frac{j}{n}\big]$ for $j=2,3,\ldots,n$.\\
    \quad $\rhd$\quad Set $t:=1$, and observe the initial state $s_1$ and interval $I(s_1)$.\smallskip
    \FOR {episodes $k=1,2,\ldots$} \smallskip
        \STATE $\rhd$\quad Let $N_{k}\left (I_j,a \right)$ be the number of times action~$a$ has been chosen in a state $\in I_j$ 
               \textit{prior} to episode~$k$, and $v_k(I_j,a)$ the respective counts \textit{in} episode $k$. \smallskip
        \STATE \textbf{Initialize episode }$k$:\\ 
             $\rhd$\quad Set the start time of episode $k$, $t_k:=t$.\\
                $\rhd$\quad Compute estimates $\rh_{k}({s,a})$ and $\ph\agg_{k}({I_i}|{s},{a})$ for rewards and transition probabilities, 
                using all samples from states in the same interval $I(s)$, respectively.\smallskip
        \STATE \textbf{Compute policy} $\tpi{k}$:\\ 
		$\rhd$\quad Let $\mathcal M_k$ be the set of plausible MDPs $\tM$ 
		with $H(\tM)\leq H$ and rewards $\rt(s, a)$ and 
	        transition probabilities $\pt(\cdot|s,a)$ satisfying
 	  \begin{eqnarray}\textstyle
          \label{eq:civR}
             \big\vert
              \rt(s, a) - \rh_k(s,a)
            \big \vert \;\;
          & \leq & \;
           L n^{-\alpha} + \sqrt{ \tfrac{7\log\left( 2 n A t_k / \delta  \right )}
              {2\max\{1,N_k(I(s),a)\}}
            }\,,
          \\
          \label{eq:civ}
           \Big\Vert
             \pt\agg(\cdot|s,a) - \ph_k\agg(\cdot|s,a)
            \Big \Vert_1
          & \leq &\;
             L n^{-\alpha} + \sqrt{ \tfrac{56n\log\left( 2 A t_k / \delta  \right )}
              {\max\{1,N_k(I(s),a)\}}
            }
          \;.
          \end{eqnarray}
	$\rhd$\quad Choose policy $\tpi{k}$ and $\tM_k\in \mathcal M_k$ such that
	\begin{equation}\label{eq:compute}
	 \rho_{\tpi{k}}(\tM_k) = \arg\max\{ \rho^*(M) \,|\, M\in \mathcal M_k \}.
	\end{equation}

%%%%%%%%%
        \STATE \textbf{Execute policy $\tpi{k}$}:
            \STATE \textbf{ while} $v_k(I(s_t),\tpi{k}(s_t)) < \max\{1,N_{k}(I(s_{t}),\tpi{k}(s_{t}))\}$ \textbf{do} \smallskip
            \begin{itemize}
                 \item[$\rhd$] Choose action $a_{t} = \tpi{k}(s_{t})$, 
                 obtain reward $r_t$, and observe next state~$s_{t+1}$.	\smallskip
                 \item[$\rhd$]  Set $t:=t+1$.\smallskip
            \end{itemize}
             \STATE \textbf{ end while}\smallskip
        \ENDFOR
    \end{algorithmic}
\end{algorithm}
It maintains a set of plausible MDPs~$\mathcal M$ and chooses optimistically an MDP $\tM\in \mathcal M$
and a policy $\tilde{\pi}$ such that the average reward $\rho_{\tilde{\pi}}(\tM)$ is maximized, cf.~\eqref{eq:compute}. 
Whereas for UCRL2 and REGAL the set of plausible MDPs is defined by confidence intervals for rewards
and transition probabilities for each individual state-action pair, for \ucrl\ we assume an MDP to be 
plausible if its \textit{aggregated} rewards and transition probabilities are within a certain range.
This range is defined by the aggregation error (determined by the assumed H\"older continuity) and 
respective confidence intervals, cf.~\eqref{eq:civR}, \eqref{eq:civ}. Correspondingly, the estimates
for rewards and transition probabilities for some state action-pair $(s,a)$ are calculated from all 
sampled values of action $a$ in states close to $s$.

More precisely, for the aggregation \ucrl\ partitions 
the state space into intervals $I_1:=\big[0, \frac{1}{n}\big]$, 
$I_k:=\big(\frac{k-1}{n}, \frac{k}{n}\big]$ for $k=2,3,\ldots,n$.
The corresponding aggregated transition probabilities are defined by
\begin{equation}
 p\agg(I_j|s,a) := \int_{I_j} p(ds'|s,a).
\end{equation}
Generally, for a (transition) probability distribution $p(\cdot)$ over $\sS$
we write $p\agg(\cdot)$ for the aggregated probability distribution with respect to $\{I_1,I_2\ldots,I_{n}\}$.
Now, given the aggregated state space $\{I_1,I_2\ldots,I_{n}\}$, estimates $\rh(s,a)$ and $\ph\agg(\cdot|s,a)$ are 
calculated from all samples of action $a$ in states in $I(s)$, the interval $I_j$ containing $s$.
(Consequently, the estimates are the same for states in the same interval.)

As UCRL2 and REGAL, \ucrl\ proceeds in episodes in which the chosen policy remains fixed.
Episodes are terminated when the number of times an action has been sampled from some interval~$I_j$
has been doubled. Only then estimates are updated and a new policy is calculated.

Since all states in the same interval $I_j$ have the same confidence intervals, finding the optimal pair~$\tM_k,\tpi{k}$ 
in~\eqref{eq:compute} is equivalent to finding the respective optimistic discretized MDP 
$\tM_k\agg$ and an optimal policy $\tpi{k}\agg$ on $\tM_k\agg$. 
Then $\tpi{k}$ can be set to be the extension of $\tpi{k}\agg$ to $\mathcal S$, that is, $\tpi{k}(s):= \tpi{k}\agg(I(s))$ for all $s$.
However, due to the constraint on the bias even in this finite case efficient 
computation of $\tM_k\agg$ and $\tpi{k}\agg$ is still an open problem.
We note that the \mbox{REGAL.C} algorithm~\cite{regal} selects optimistic MDP and optimal policy in the same way as \ucrl.

While the algorithm presented here is the first modification of UCRL2 to continuous reinforcement learning problems,
there are similar adaptations to online aggregation~\cite{or-adagg} and learning in finite state MDPs 
with some additional similarity structure known to the learner \cite{or-restless}.

%%%%
\section{Regret Bounds}
%%%%
For \ucrl\ we can derive the following bounds on the regret.

\begin{theorem}\label{thm}
 Let $M$ be an MDP with continuous state space $[0,1]$, $A$ actions, rewards and transition probabilities 
 satisfying Assumptions~\ref{ass:rew} and~\ref{ass:prob}, and bias span upper bounded by $H$. 
Then with probability $1-\delta$, the regret of \ucrl\ (run with input parameters $n$ and $H$) after $T$ steps is upper bounded by
\begin{equation}\label{eq:thm}
   const \cdot n H \sqrt{A T \log\big(\tfrac{T}{\delta}\big)} +  const'\cdot H L n^{-\alpha} T.
\end{equation}
Therefore, setting $n=T^{1/(2+2\alpha)}$ gives regret upper bounded by 
\[
 const\cdot H L \sqrt{A\log\big(\tfrac{T}{\delta}\big)} \cdot T^{(2+\alpha)/(2+2\alpha)}.
\]
With no known upper bound on the bias span, guessing $H$ by $\log T$
one still obtains an upper bound on the regret of $\tilde{O}(T^{(2+\alpha)/(2+2\alpha)})$.
\end{theorem}

Intuitively, the second term in the regret bound of \eqref{eq:thm} is the discretization error, 
while the first term corresponds to the regret on the discretized MDP. A detailed proof of 
Theorem~\ref{thm} can be found in Section~\ref{sec:proof} below.

\begin{remark}[\textbf{$d$-dimensional case}] \label{rem:d}
 The general $d$-dimensional case can be handled as described for dimension 1,
 with the only difference being that the discretization now has $n^d$ states, so that
 one has $n^d$ instead of $n$ in the first term of~\eqref{eq:thm}. Then choosing
 $n=T^{1/(2d+2\alpha)}$ bounds the regret by $\tilde{O}(T^{(2d+\alpha)/(2d+2\alpha)})$.
\end{remark}

\begin{remark}[\textbf{unknown horizon}] 
 If the horizon $T$ is unknown then the doubling trick (executing the algorithm in rounds $i=1,2,\ldots$ guessing $T=2^i$
 and setting the confidence parameter to $\delta/2^i$) gives the same bounds.
\end{remark}

\begin{remark}[\textbf{unknown  H\"older  parameters}]\label{r:uh} The \ucrl\ algorithm receives (bounds on) the  H\"older  parameters $L$ as $\alpha$ as inputs.
If these parameters are not known, then one can still obtain sublinear regret bounds albeit with worse dependence on $T$. 
Specifically, 
we can use the model-selection technique introduced in  \cite{maimury}. To do this, fix a certain  number $J$ of values for the constants $L$ and~$\alpha$;
each of these values will be considered as a model.
The model selection consists in running \ucrl\ with each of these parameter values for a certain period of $\tau_0$ time steps (\textit{exploration}).
Then one selects the model  with the highest reward and uses it for a period of $\tau'_0$ time steps (\textit{exploitation}), while checking that its average reward
stays  within~(\ref{eq:thm}) of what was obtained in the exploitation phase. If the average reward does not pass this test, then the model with the second-best average reward is selected, and so on. 
Then one switches to exploration with longer periods $\tau_1$, etc.
 Since there are no guarantees on the behavior of \ucrl\ when  the  H\"older  parameters are wrong,  none of the models can be discarded at any stage.
Optimizing over the parameters $\tau_i$ and $\tau'_i$ as done in~\cite{maimury}, and increasing the number $J$ of considered parameter values, 
one can obtain regret bounds of $\tilde{O}(T^{(2+2\alpha)/(2+3\alpha)})$, or $\tilde{O}(T^{4/5})$ in the Lipschitz case. For details see \cite{maimury}.
 Since in this model-selection process \mbox{\ucrl} is used in  a ``black-box'' fashion,
the exploration is rather wasteful, and thus we think that this bound is suboptimal. 
 Recently, the results of~\cite{maimury} have been improved~\cite{icml13}, and it seems that similar analysis 
gives improved regret bounds for the case of unknown H\"older parameters as well.
\end{remark}

The following is a complementing lower bound on the regret for continuous state reinforcement learning.
 
\begin{theorem}\label{thm:lower}
For any $A, H>1$ and any reinforcement learning algorithm
there is a continuous state reinforcement learning problem 
with $A$ actions and bias span $H$
satisfying Assumption~\ref{ass:rew} 
such that the algorithm suffers regret of $\,\Omega(\sqrt{H A T})$.
\end{theorem}
\begin{proof}
Consider the following reinforcement learning problem with state space $[0,1]$.
The state space is partitioned into $n$ intervals $I_j$ of equal size. The transition 
probabilities for each action~$a$ are on each of the intervals~$I_j$ concentrated and
equally distributed on the same interval~$I_j$. The rewards on each 
interval~$I_j$ are also constant for each $a$ and are chosen
as in the lower bounds for a multi-armed bandit problem \cite{acfs} with $n A$ arms.
That is, giving only one arm slightly higher reward, it is known~\cite{acfs}
that regret of $\Omega(\sqrt{nAT})$ can be forced upon any algorithm on the respective bandit problem.
Adding another action giving no reward and equally distributing over the whole
state space, the bias span of the problem is $n$ and the regret $\Omega(\sqrt{H AT})$.
\end{proof}

\begin{remark}\label{rem:lower}
Note that Assumption~\ref{ass:prob} does not hold in the example used in the proof of Theorem~\ref{thm:lower}.
However, the transition probabilities are piecewise constant (and hence Lipschitz) and known to the learner.
Actually, it is straightforward to deal with piecewise H\"older continuous rewards and transition probabilities 
where the finitely many points of discontinuity are known to the learner. If one makes sure that 
the intervals of the discretized state space do not contain any discontinuities, it is easy to adapt \ucrl\ 
and Theorem~\ref{thm} accordingly.
\end{remark}

\begin{remark}[\textbf{comparison to bandits}]
 The bounds of Theorems~\ref{thm} and \ref{thm:lower} cannot be directly compared to bounds 
 for the continuous-armed bandit problem~\cite{kleinberg,auorsze,upfal,remi-x}, because
 the latter is no special case of learning MDPs with continuous state space
 (and rather corresponds to a continuous action space). Thus, in particular one cannot
 freely sample an arbitrary state of the state space as assumed in continuous-armed bandits.
\end{remark}

%%%%%
\section{Proof of Theorem~\ref{thm}}\label{sec:proof}
%%%% 

For the proof of the main theorem we adapt the proof of the regret bounds for finite MDPs in~\cite{jaorau} 
and~\cite{regal}. Although the state space is now continuous, due to the finite horizon $T$, we can reuse some arguments,
so that we keep the structure of the original proof of Theorem~2 in~\cite{jaorau}.
Some of the necessary adaptations made are similar to techniques used for showing regret bounds 
for other modifications of the original UCRL2 algorithm \cite{or-adagg,or-restless},
which however only considered finite-state MDPs.

%%%
\subsection{Splitting into Episodes}
Let $v_k(s,a)$ be the number of times action $a$ has been chosen in episode~$k$ when being in state $s$,
and denote the total number of episodes by $m$.
Then setting $\Delta_k:= \sum_{s,a} v_k(s,a)  ( \rho^* -  \meanr(s,a))$,
with probability at least $1-\tfrac{\delta}{12T^{5/4}}$ the regret of \ucrl\ after $T$ steps is upper bounded by 
(cf.~Section 4.1 of \cite{jaorau}),
  \begin{equation}\label{eq:r-bound}
    \st  {\sqrt{\tfrac{5}{8} T \log \left(\tfrac{8T}{\delta}\right)}} + \sum_{k=1}^m \Delta_k \;.
  \end{equation}

%%%
\subsection{Failing Confidence Intervals}
Next, we consider the regret incurred when the true MDP~$M$ is not contained in 
the set of plausible MDPs~${\mathcal{M}_{k}}$. Thus, fix a state-action pair 
$(s,a)$, and recall that $\hat{r}(s,a)$ and $\hat{p}\agg(\cdot|s,a)$ are the 
estimates for rewards and transition probabilities calculated from all samples 
of state-action pairs contained in the same interval~$I(s)$. 
Now assume that at step $t$ there have been $N>0$ samples of action $a$ in states in~$I(s)$ 
and that in the $i$-th sample a transition from state $s_i\in I(s)$ to state $s'_i$
has been observed $(i=1,\ldots,N)$.
  
First, concerning the rewards one obtains as in the proof of Lemma 17 in Appendix C.1 of \cite{jaorau} --- 
but now using Hoeffding for independent and not necessarily identically distributed random variables ---  that 
  \begin{eqnarray}\label{eq:rewc}
  \st  \Pr\left\{
      \left\vert \vphantom{X^X_X}
        \hat{r}(s,a) - \mathbb{E}[\hat{r}(s,a)]
      \right\vert
      \geq \sqrt{\frac{7}{2N}\log\big( \tfrac{2 nA t}{\delta}\big) }
    \right\}
  & \leq & \frac{\delta}{60 n A t^7}.
  \end{eqnarray}

Concerning the transition probabilities, we have for a suitable $x\in\{-1,1\}^{n}$
\begin{eqnarray}
	\lefteqn{\Big\| \hat{p}\agg(\cdot|s,a)- \mathbb{E}[\hat{p}\agg(\cdot|s,a)] \Big\|_1  
	=
	\sum_{j=1}^n \Big|  \hat{p}\agg(I_j|s,a) - \mathbb{E}[\hat{p}\agg(I_j|s,a)]  \Big|}  \nonumber
	\\
	&=&  
	 \sum_{j=1}^n \Big(  \hat{p}\agg(I_j|s,a) - \mathbb{E}[\hat{p}\agg(I_j|s,a)]  \Big)\, x(I_j)  \nonumber\\
	&=&
	 \tfrac{1}{N}\sum_{i=1}^N 
			\Big( x(I(s'_i))  - \int_{\sS} p(ds'|s_i,a)\cdot x(I(s'))  \Big)\,. \quad \label{eq:ah}
\end{eqnarray}
For any $x\in\{-1,1\}^{n}$, $X_i := x(I(s'_i))  - \int_{\sS} p(ds'|s_i,a)\cdot x(I(s'))$ 
is a  martingale difference sequence 
 with $|X_i|\leq 2$, so that by Azuma-Hoeffding inequality (e.g., Lemma 10 in \cite{jaorau}),
$\Pr\{\vphantom{X^X_X} \sum_{i=1}^N X_i \geq \theta\} \leq  \exp( -  \theta^2/8N)$ and in particular 
\[
   {\st \Pr\Big\{\vphantom{X^X_X} \sum_{i=1}^N X_i \geq \sqrt{56 n N \log\big( \tfrac{2 A t}{\delta}\big)} \Big\} 
    \leq  \Big(\tfrac{\delta}{2 A t}\Big)^{7n}}
    \leq  \frac{\delta}{2^n 20 n A t^7}.
\]
A union bound over all sequences $x\in\{-1,1\}^n$ then yields from \eqref{eq:ah} that
 \begin{eqnarray}\label{eq:probc}
   \st \Pr\left\{
      \Big\| \vphantom{X^X_X}
         \hat{p}\agg(\cdot|s,a) - \mathbb{E}[\hat{p}\agg(\cdot|s,a)]
      \Big\|_1
      \geq \sqrt{\frac{56 n}{N}\log\big( \tfrac{2 A t}{\delta}\big) }
    \right\}
  &  \leq & \frac{\delta}{20 nA t^7}.
  \end{eqnarray}

Another union bound over all $t$ possible values for $N$, all $n$ intervals and all actions
shows that the confidence intervals in \eqref{eq:rewc} and \eqref{eq:probc} hold at time $t$ with probability at least $1-\frac{\delta}{15 t^{6}}$
for the actual counts $N(I(s),a)$ and all state-action pairs $(s,a)$. (Note that the equations 
\eqref{eq:rewc} and \eqref{eq:probc} are the same for state-action pairs with states in the same interval.)

Now, by linearity of expectation $\mathbb{E}[\hat{r}(s,a)]$ can be written as $\frac{1}{N} \sum_{i=1}^N r(s_i,a)$. 
Since the $s_i$ are assumed to be in the same interval $I(s)$, it follows that $|\mathbb{E}[\hat{r}(s,a)] - r(s,a)| < L n^{-\alpha}$.
Similarly, $\big\|\mathbb{E}[\hat{p}\agg(\cdot|s,a)]-p\agg(\cdot|s,a)\big\|_1 < L n^{-\alpha}$. Together with \eqref{eq:rewc} and \eqref{eq:probc}
this shows that with probability at least $1-\frac{\delta}{15 t^{6}}$ for all state-action pairs $(s,a)$ 
  \begin{eqnarray}
      \left\vert \vphantom{X^X_X}
        \hat{r}(s,a) - r(s,a)
      \right\vert
      &<& L n^{-\alpha} + \st\sqrt{\frac{7 \log(2nAt/\delta)}{2\max\{1,N(I(s),a)\}} }  \label{eq:rewc2}\,,  \\
       \Big\| \vphantom{X^X_X}
         \hat{p}\agg(\cdot|s,a) - {p}\agg(\cdot|s,a)
      \Big\|_1
      &<& L n^{-\alpha} + \st \sqrt{\frac{56n \log(2At/\delta)}{\max\{1,N(I(s),a)\}}}\,. \label{eq:probc2}
  \end{eqnarray} 
This shows that the true MDP is contained in the set of plausible MDPs $\mathcal{M}(t)$ at step $t$
 with probability at least $1-\frac{\delta}{15 t^{6}}$,
just as in Lemma 17 of~\cite{jaorau}. The argument that 
  \begin{equation} \label{eq:confidence}
    \sum_{k=1}^m \Delta_k \Ind{M \not\in {\mathcal{M}_{k}}}  \; \leq \; \sqrt{T}
  \end{equation}
   with probability at least $1-\frac{\delta}{12 T^{5/4}}$ then can be taken without any changes from 
   Section 4.2 of \cite{jaorau}.

%%%
\subsection{Regret in Episodes with $M\in{\mathcal{M}_{k}}$}\label{sec:confhold}

Now for episodes with $M\in{\mathcal{M}_{k}}$, by the optimistic choice of $\tM_k$ and $\tpi{k}$ in~\eqref{eq:compute} we can bound 
\begin{eqnarray*}
 \Delta_k   &=&   \sum_{s} v_k(s,\tpi{k}(s))  \big( \rho^* -  \meanr(s,\tpi{k}(s))\big) \\
        &\leq&  
        \sum_{s} v_k(s,\tpi{k}(s))  \big( \rhot_k^* -  \meanr(s,\tpi{k}(s))\big)  \\ 		
 	 &=&   \sum_{s} v_k(s,\tpi{k}(s))  \big( \rhot_k^* -  \rt_k(s,\tpi{k}(s))\big)  
 	      +  \sum_{s} v_k(s,\tpi{k}(s))  \big(  \rt_k(s,\tpi{k}(s))  -  r(s,\tpi{k}(s))  \big).
\end{eqnarray*}

Any term $\tilde{r}_k(s,a) -\meanr(s,a) \leq |\tilde{r}_k(s,a) -\hat{r}_k(s,a)| + |\hat{r}_k(s,a)-\meanr(s,a)|$ 
is bounded according to \eqref{eq:civR} and \eqref{eq:rewc2}, as we assume that $\tMk,M\in{\mathcal{M}_{k}}$, so 
that summarizing states in the same interval $I_j$
\begin{equation*}
\Delta_k   \leq  \sum_{s} v_k(s,\tpi{k}(s))  \big( \rhot_k^* -  \rt_k(s,\tpi{k}(s))\big) 
        + 2 \sum_{j=1}^n \sum_{a\in \sA} v_k(I_j,a) \left(Ln^{-\alpha} + \sqrt{ \tfrac{7\log\left( 2 n A t_k / \delta \right )}
              {2\max\{1,N_k(I_j,a)\}} }  \right).
\end{equation*}
Since $\max\{1,N_k(I_j,a)\} \leq t_k \leq T$, setting $\tau_k:=t_{k+1}-t_k$ to be the length of episode $k$ we have \vspace{0.5mm}
\begin{eqnarray}
\Delta_k &\leq&   \sum_{s} v_k(s,\tpi{k}(s))  \big( \rhot_k^* -  \rt_k(s,\tpi{k}(s))\big)  \nonumber\\[-1.5mm]
        && + \; 2L n^{-\alpha} \tau_k + \sqrt{ 14\log\left( \tfrac{2 n A T}{\delta}  \right )}  
        \sum_{j=1}^n \sum_{a\in \sA} \frac{ v_k(I_j,a)}{\sqrt{\max\{1,N_k(I_j,a)\}} }\,.\quad  \label{eq:poisson3}
\end{eqnarray}

We continue analyzing the first term on the right hand side of \eqref{eq:poisson3}.
By the Poisson equation~\eqref{eq:poisson} for $\tpi{k}$ on $\tM_k$, denoting the respective bias by $\tlambda_k:=\tlambda_{\tpi{k}}$ 
we can write\vspace{0.5mm}
\begin{eqnarray}
  \lefteqn{\sum_{s}}&&\!\!\!\!\!\! v_k(s,\tpi{k}(s)) \big( \rhot_k^* -  \rt_k(s,\tpi{k}(s))\big) \nonumber \\[-1mm]	
 	    &=&  \sum_{s} v_k(s,\tpi{k}(s))  \Big( \int_\sS \pt_k(ds'|s,\tpi{k}(s)) \cdot \tlambda_k(s') - \tlambda_k(s) \Big) \nonumber \\ 	
 \nonumber \\
 	      &=&  \sum_{s} v_k(s,\tpi{k}(s))  \Big( \int_{\sS} p(ds'|s,\tpi{k}(s)) \cdot \tlambda_k(s') - \tlambda_k(s) \Big) \label{eq:neu1} \\
 	      && + \sum_{s} v_k(s,\tpi{k}(s)) \sum_{j=1}^n\int_{I_j} \Big( \pt_k(ds'|s,\tpi{k}(s)) -p(ds'|s,\tpi{k}(s)) \Big) \cdot \tlambda_k(s'). \label{eq:neu2}
\end{eqnarray}

%%%
\subsection{The True Transition Functions}\label{sec:MainTerm}
Now $\big\| \pt\agg_k(\cdot|s,a) -p\agg(\cdot|s,a)\big\|_1 
           \leq 
             \big\| \pt\agg_k(\cdot|s,a) -\ph_k\agg(\cdot|s,a)\big\|_1  +  \big\| \ph\agg_k(\cdot|s,a) -p\agg(\cdot|s,a)\big\|_1$
can be bounded by \eqref{eq:civ} and \eqref{eq:probc2}, because we assume $\tMk, M \in {\mathcal{M}_{k}}$. 
Hence, since by definition of the algorithm $H$ bounds the bias function $\tlambda_k$, the term in~\eqref{eq:neu2} is bounded by
\begin{eqnarray}
 \lefteqn{\sum_{s}}&&\!\! v_k(s,\tpi{k}(s)) \sum_{j=1}^n\int_{I_j} \tlambda_k(s')\Big( \pt_k(ds'|s,\tpi{k}(s)) -p(ds'|s,\tpi{k}(s)) \Big) 
 \nonumber \\
  &&\leq\;
  \sum_{s} v_k(s,\tpi{k}(s))\cdot H\cdot \sum_{j=1}^n \Big( \pt\agg_k(I_j|s,\tpi{k}(s)) -p\agg(I_j|s,\tpi{k}(s)) \Big)  \nonumber \\
  &&\leq\;
  \sum_{s} v_k(s,\tpi{k}(s)) \cdot H \cdot  
  2 \left( L n^{-\alpha} 
  + \st \sqrt{
        \frac{56 n \log\left(2 A T / \delta \right )}{\max\{1,N_k(I(s),a_t)\}}
      } \right)\nonumber \\
      &&=\; 
      2 HL n^{-\alpha} \tau_k + 4H \sqrt{14n\log\left( \tfrac{2 A T}{\delta}  \right )}
       \sum_{j=1}^n \sum_{a\in \sA} \frac{ v_k(I_j,a)}{\sqrt{\max\{1,N_k(I_j,a)\}} }\,
    , \qquad\label{eq:zwi}
\end{eqnarray}
while for the term in~\eqref{eq:neu1}
\begin{eqnarray*}
\lefteqn{\sum_{s}}&&\!\! v_k(s,\tpi{k}(s))  \Big( \int_{\sS} p(ds'|s,\tpi{k}(s)) \cdot \tlambda_k(s') - \tlambda_k(s) \Big) \\
 &&=\sum_{t=t_k}^{t_{k+1}-1}  \Big( \int_{\sS} p(ds'|s_t,a_t)\cdot \tlambda_k(s') -  \tlambda_k(s_t) \Big) \\
 &&=\sum_{t=t_k}^{t_{k+1}-1}  \Big( \int_{\sS} p(ds'|s_t,a_t)\cdot \tlambda_k(s') -  \tlambda_k(s_{t+1}) \Big) 
                     +  \tlambda_k(s_{t_{k+1}})  -  \tlambda_k(s_{t_k}). 
\end{eqnarray*}

Let $k(t)$ be the index of the episode time step $t$ belongs to.
Then the sequence $X_t:=\int_{\sS} p(ds'|s_t,a_t)\cdot \tlambda_{k(t)}(s') -  \tlambda_{k(t)}(s_{t+1})$
is a sequence of martingale differences so that Azuma-Hoeffding inequality shows (cf. Section 4.3.2 
and in particular eq.~(18) in~\cite{jaorau}) that after summing over all episodes we have 
\begin{eqnarray}
   \sum_{k=1}^m 
   \bigg( 
     \sum_{t=t_k}^{t_{k+1}-1}  \Big( \int_{\sS} p(ds'|s_t,a_t)\cdot \tlambda_k(s') -  \tlambda_k(s_{t+1}) \Big) 
                     +  \tlambda_k(s_{t_{k+1}})  -  \tlambda_k(s_{t_k})
   \bigg)\nonumber\\
   \leq  
   H \sqrt{\tfrac{5}{2}T\log\left(\tfrac{8T}{\delta}\right)} 
           +  H n A \log_2\big(\tfrac{8T}{n A}\big), \label{eq:counts}
\end{eqnarray}
where the second term comes from an upper bound on the number of episodes, which can be derived 
analogously to Appendix~C.2 of \cite{jaorau}.

\subsection{Summing over Episodes with $M\in{\mathcal{M}_{k}}$}
\label{sec:together}
To conclude, we sum \eqref{eq:poisson3} over all the
  episodes with ${\M\in{\mathcal{M}_{k}}}$, using \eqref{eq:neu1}, \eqref{eq:zwi}, and \eqref{eq:counts}.
  This yields that with probability at least $1 -\tfrac{\delta}{12 T^{5/4}}$
  \begin{eqnarray} 
   \lefteqn{}&&   \sum_{k=1}^m \Delta_k\Ind{\M\in\sM_{k}}
       \leq  \label{eq:CiHold}
         2HL n^{-\alpha} T   +    4 H \sqrt{14n\log\left( \tfrac{2 A T}{\delta}  \right )}\cdot 
                  \sum_{k=1}^m \sum_{j=1}^n\sum_{a\in\sA} \frac{ v_k(I_j,a)}{\sqrt{\max\{1,N_k(I_j,a)\}} }
      \nonumber  \\ & & \mbox{}
        + H \sqrt{\tfrac{5}{2}T \log\left(\tfrac{8T}{\delta}\right)}
        +  H n A \log_2\left(\tfrac{8T}{n A}\right)  \nonumber\\ 
 & & \mbox{}
        + 2L n^{-\alpha} T + \sqrt{ 14\log\left( \tfrac{2 n A T}{\delta}  \right )}
        \sum_{k=1}^m \sum_{j=1}^n\sum_{a\in \sA} \frac{ v_k(I_j,a)}{\sqrt{\max\{1,N_k(I_j,a)\}} }
        .\quad
  \end{eqnarray}
Analogously to Section 4.3.3 and Appendix C.3 of \cite{jaorau}, one can show that
  \begin{equation*}
    \sum_{j=1}^n\sum_{a\in \sA}  \sum_k  \frac{v_k(I_j,a)}{\sqrt{\max\{1,N_k(I_j,a)\}} } \;\leq\;  \big(\sqrt{2}+1\big)\sqrt{n A T},
  \end{equation*}
and we get from \eqref{eq:CiHold} after some simplifications 
that with probability $\geq 1 -\tfrac{\delta}{12 T^{5/4}}$
  \begin{eqnarray} 
     \lefteqn{}&& \sum_{k=1}^m \Delta_k\Ind{\M\in\sM_{k}}
       \leq 
         H \sqrt{\tfrac{5}{2}T \log\left(\tfrac{8T}{\delta}\right)}
        + H n A \log_2\left(\tfrac{8T}{n A}\right) 
         \nonumber \\ & & \mbox{}
        + \Big( (4H + 1) \sqrt{14n\log\left( \tfrac{2 A T}{\delta} \right )} \Big) \big(\sqrt{2}+1\big)\sqrt{n A T}
         + 2(H+1)Ln^{-\alpha} T
        \;.  \label{eq:CiHold2}
  \end{eqnarray}
Finally, evaluating \eqref{eq:r-bound} by
  summing $\Delta_k$ over all episodes, by \eqref{eq:confidence} and \eqref{eq:CiHold2} we have with probability $\geq 1 -\tfrac{\delta}{4 T^{5/4}}$ an upper bound on the regret of 
  \begin{eqnarray*} 
     \lefteqn{}    && \sqrt{\tfrac{5}{8} T \log\left(\tfrac{8T}{\delta}\right)} + \sum_{k=1}^m  \Delta_k\Ind{\M\notin\sM_{k}}
        + \sum_{k=1}^m \Delta_k\Ind{\M\in\sM_{k}}
      \label{eq:combined1} \\
      && \leq  
           \sqrt{\tfrac{5}{8}T \log\left(\tfrac{8T}{\delta}\right)} + \sqrt{T} 
        + H \sqrt{\tfrac{5}{2}T \log\left(\tfrac{8T}{\delta}\right)}
        + H n A \log_2\left(\tfrac{8T}{n A}\right) 
        \\ & & \mbox{}
        + \Big( (4H + 1) \sqrt{14n\log\left( \tfrac{2 A T}{\delta}  \right )} \Big) \big(\sqrt{2}+1\big)\sqrt{n A T}
         + 2(H+1)L n^{-\alpha} T . \quad \nonumber
  \end{eqnarray*}
A union bound over all possible values of $T$ and further simplifications as in Appendix C.4 of \cite{jaorau} finish the proof. \qed

\section{Outlook}

We think that a generalization of our results to continuous action space should not pose any major problems.
In order to improve over the given bounds, it may be promising to investigate more sophisticated discretization patterns.

The assumption of H\"older continuity is an obvious, yet not the only possible assumption
one can make about the transition probabilities and reward functions. A more general problem is to assume
a set $\mathcal F$ of functions, find a way to measure the ``size'' of $\mathcal F$, and derive 
regret bounds depending on this size of $\mathcal F$.

\subsubsection*{Acknowledgments}
{ The authors would like to thank the three anonymous reviewers for their helpful suggestions 
and R{\'e}mi Munos for useful discussion which helped to improve the bounds.
This research 
was funded by the Ministry of Higher Education and Research, Nord-Pas-de-Calais Regional Council and 
 FEDER (Contrat de Projets Etat Region CPER  2007-2013),
  ANR  projects EXPLO-RA (ANR-08-COSI-004),
  Lampada (ANR-09-EMER-007)  and CoAdapt,
  and by the European Community's  FP7 Program  under grant agreements 
 n$^\circ$\,216886 (PASCAL2) and n$^\circ$\,270327 (CompLACS).
 The first author is currently funded by the Austrian Science Fund (FWF): J~3259-N13.}

\end{document}